\documentclass{article}
\usepackage{iclr2027_conference,times}

\usepackage{amsmath,amsfonts,bm}

\def\eqref#1{equation~\ref{#1}}

\def\1{\bm{1}}

\DeclareMathAlphabet{\mathsfit}{\encodingdefault}{\sfdefault}{m}{sl}
\SetMathAlphabet{\mathsfit}{bold}{\encodingdefault}{\sfdefault}{bx}{n}

\usepackage{amsmath,amssymb,amsthm}
\usepackage{array}
\usepackage{booktabs,longtable}
\usepackage{multirow}
\usepackage{float,placeins}
\usepackage{xcolor}
\usepackage{tikz}
\usetikzlibrary{positioning,arrows.meta,decorations.pathreplacing}
\usepackage{hyperref}
\usepackage{url}
\usepackage{microtype}

\definecolor{fgBlue}{HTML}{377EB8}
\definecolor{fgOrange}{HTML}{E69F00}
\definecolor{fgGreen}{HTML}{009E73}
\definecolor{fgPurple}{HTML}{7B3294}

\newtheorem{proposition}{Proposition}
\newcommand{\method}{\textsc{ArcFlow}}

\newcommand{\ScaleTasks}{21}

\newcommand{\ScaleTrainNFE}{4}
\newcommand{\ScaleOuterIterations}{64}
\newcommand{\ScaleInteractions}{$62.9$~million}

\newcommand{\TaskCountEligible}{30}
\newcommand{\TaskCountBestOne}{25}

\newcommand{\TaskCountDMOne}{18}
\newcommand{\TaskCountPandaOne}{3}
\newcommand{\TaskCountLocoOne}{4}

\newcommand{\TableThreeBCOne}{$56.2$}
\newcommand{\TableThreeBCTwo}{$56.2$}
\newcommand{\TableThreeBCFour}{$59.7$}
\newcommand{\TableThreeBCEight}{$65.0$}
\newcommand{\TableThreeBCSixteen}{$70.6$}
\newcommand{\TableThreeBCWorst}{$56.2$}
\newcommand{\TableThreeBCMean}{$61.5$}
\newcommand{\TableThreeReinFlowOne}{$1.5$}
\newcommand{\TableThreeReinFlowTwo}{$71.7$}
\newcommand{\TableThreeReinFlowFour}{$92.5$}
\newcommand{\TableThreeReinFlowEight}{$92.4$}
\newcommand{\TableThreeReinFlowSixteen}{$93.8$}
\newcommand{\TableThreeReinFlowWorst}{$1.5$}
\newcommand{\TableThreeReinFlowMean}{$70.4$}
\newcommand{\TableThreeOneStepOne}{$98.0$}
\newcommand{\TableThreeOneStepTwo}{$83.3$}
\newcommand{\TableThreeOneStepFour}{$78.3$}
\newcommand{\TableThreeOneStepEight}{$74.6$}
\newcommand{\TableThreeOneStepSixteen}{$83.9$}
\newcommand{\TableThreeOneStepWorst}{$74.6$}
\newcommand{\TableThreeOneStepMean}{$83.6$}
\newcommand{\TableThreeCyclingOne}{$85.3$}
\newcommand{\TableThreeCyclingTwo}{$96.7$}
\newcommand{\TableThreeCyclingFour}{$98.2$}
\newcommand{\TableThreeCyclingEight}{$78.3$}
\newcommand{\TableThreeCyclingSixteen}{$88.9$}
\newcommand{\TableThreeCyclingWorst}{$78.3$}
\newcommand{\TableThreeCyclingMean}{$89.5$}
\newcommand{\TableThreeArcFlowOne}{$91.1$}
\newcommand{\TableThreeArcFlowTwo}{$95.3$}
\newcommand{\TableThreeArcFlowFour}{$93.5$}
\newcommand{\TableThreeArcFlowEight}{$96.3$}
\newcommand{\TableThreeArcFlowSixteen}{$97.1$}
\newcommand{\TableThreeArcFlowWorst}{$91.1$}
\newcommand{\TableThreeArcFlowMean}{$94.7$}

\floatstyle{ruled}
\newfloat{algorithm}{t}{loa}
\floatname{algorithm}{Algorithm}
\newcounter{algline}
\newenvironment{pseudocode}
  {\begin{list}{\textup{\scriptsize\arabic{algline}:}}%
     {\usecounter{algline}%
      \setlength{\leftmargin}{2.0em}\setlength{\labelwidth}{1.4em}%
      \setlength{\labelsep}{0.6em}\setlength{\itemsep}{0.6pt}%
      \setlength{\topsep}{2pt}\setlength{\partopsep}{0pt}%
      \setlength{\parsep}{0pt}\setlength{\rightmargin}{0pt}}}
  {\end{list}}
\newcommand{\algin}{\hspace*{1.3em}}
\title{\method: Training Flow Policies for Multiple Solver Budgets}

\iclrfinalcopy
\author{Hieu Tran$^{1}$\quad Hong Yu$^{1,2}$\\
\textnormal{$^{1}$University of Massachusetts Amherst\quad $^{2}$University of Massachusetts Lowell}\\
\textnormal{\texttt{hieutran@umass.edu}}}

\begin{document}
\renewcommand{\topfraction}{0.92}
\renewcommand{\bottomfraction}{0.75}
\renewcommand{\textfraction}{0.06}
\renewcommand{\floatpagefraction}{0.75}
\setcounter{topnumber}{4}
\setcounter{bottomnumber}{2}
\setcounter{totalnumber}{6}
\maketitle
\lhead{Preprint}

\begin{abstract}
Flow-matching policies adjust inference cost through the number of solver
steps, but reinforcement learning at a single fixed solver budget can hide
degradation at cheaper deployment budgets and even erase useful one-step
behavior present in a pretrained policy. We introduce
\textbf{Action--Ratio--Clip Flow (\method{})}, a
deployment-aware algorithm for
training one flow policy across a specified range of solver budgets. \method{}
combines action-aligned supervision, budget exposure, and worst-budget
checkpoint selection. For reconstruction objectives, its bridge aggregates
temporal probes into one surrogate ratio and one PPO-style clip per action,
with a boundary residual exactly equal to deployed one-step reconstruction
error. Cycling exposes the shared field to the deployment range, and selection
optimizes its performance floor. For native velocity objectives, multi-budget
endpoint supervision applies the same principle without replacing the native
update. Both realizations preserve the architecture and deployment-time cost.
Across DM~Control, manipulation, and locomotion, \method{} attains the highest
mean one-step return among five flow-policy baselines on \TaskCountBestOne{} of
\TaskCountEligible{} tasks under each method's reported training
settings. Component studies identify complementary improvements from boundary
supervision, the fixed-grid bridge, budget exposure, and deployment selection.
These results establish
deployment-budget-aware training as an effective route to strong low-step
control with one policy serving multiple inference costs.
\footnote{Code can be accessed at: \url{https://anonymous.4open.science/r/arcflow-868A}}
\end{abstract}

\section{Introduction}
\label{sec:intro}

Flow matching represents expressive action distributions by learning a
velocity field that transports noise to actions
\citep{lipman2023flow,liu2023rectified}. Its solver controls inference
cost: more steps refine integration, while fewer steps reduce network
evaluations, helping robotic policies meet latency budgets.
Consistency and shortcut models enable one- and few-step generation
\citep{song2023consistency,frans2025shortcut}, creating
an opportunity to retain expressive policies at low deployment cost.

Reinforcement learning (RL) improves policies through interaction. FPO and
FPO++ derive policy-gradient surrogates from flow-matching losses
\citep{mcallister2025fpo,yi2026fpopp}; Diffusion Policy Policy Optimization
(DPPO) applies policy gradients to a stochastic diffusion denoising chain
\citep{ren2025dppo}; ReinFlow adds learnable exploration noise to flow policies
\citep{zhang2025reinflow}; and PolicyFlow approximates importance ratios from
velocity-field changes along interpolation paths
\citep{yang2026policyflow}. At deployment, solver steps determine latency and
may change the action of the same trained field.
High return at the training budget gives no assurance of high one-step return;
RL can even erode low-step behavior already present in the prior.
This gap motivates us to train and select one field across its deployment
range (Figure~\ref{fig:solver-bridge-overview}).

\begin{figure}[t]
\centering
\includegraphics[width=\textwidth]{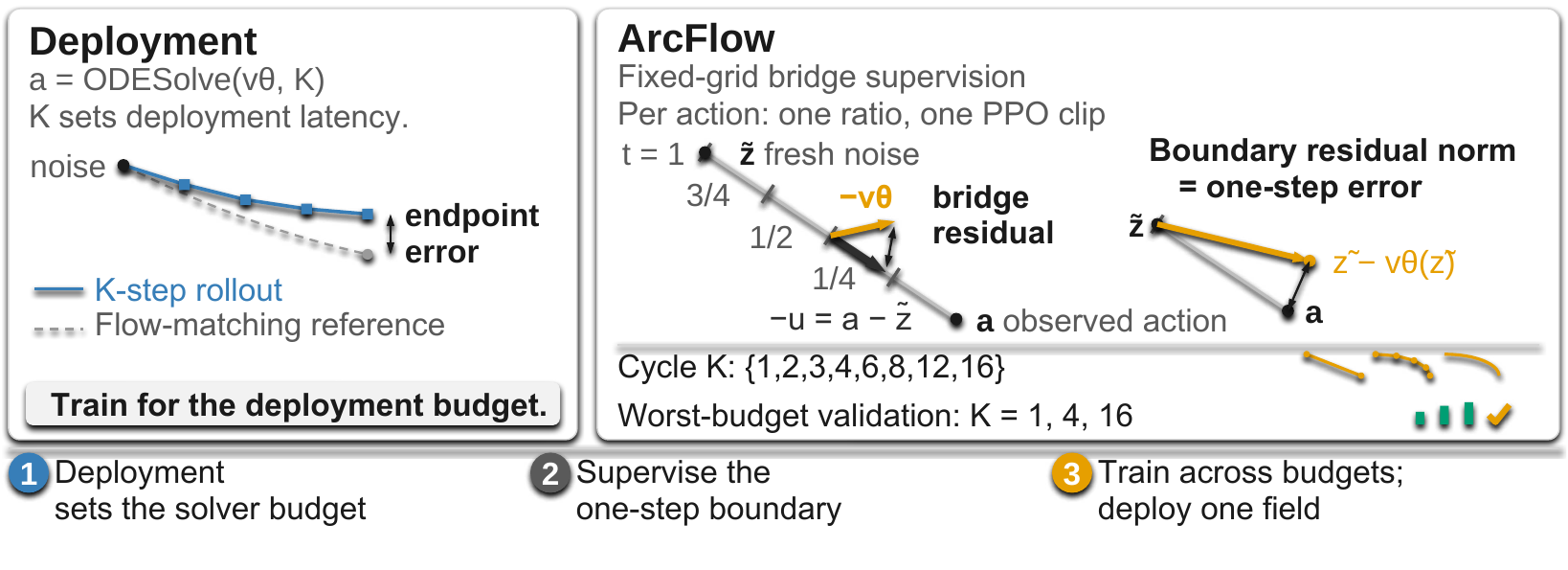}
\caption{\textbf{\method{} for multi-budget deployment.}
Left: the solver budget $K$ controls deployment cost.
Right: the bridge-based control realization uses fresh noise and the observed
action;
at the boundary, the residual norm equals one-step reconstruction error.
Probes are averaged before one surrogate ratio and PPO clip per action.
Rollout budgets cycle, and checkpoint selection maximizes the worst
validation return over $K\in\{1,4,16\}$.}
\label{fig:solver-bridge-overview}
\end{figure}

Two properties of the training objective and solver discretization help
explain this behavior. First, reconstruction-based objectives may provide
limited supervision at the deployment boundary. In FPO
\citep{mcallister2025fpo}, for example, the noise-reconstruction formulation
weights velocity error by $(1-t)^2$, so the boundary queried by a one-step
solver receives zero weight. Direct velocity supervision instead provides a
boundary-aligned signal. This observation complements prior work on
time-dependent loss weighting in generative modeling
\citep{choi2022p2,hang2023minsnr} and highlights the importance of aligning
the weighting scheme with the deployed action map. Second, because different
solver budgets induce different action maps, updates based on rollouts from a
single budget need not improve all budgets uniformly. We study this effect in
the ReinFlow \citep{zhang2025reinflow} fine-tuning pipeline, where
native velocity supervision is already available.

We introduce \textbf{Action--Ratio--Clip Flow (\method{})}, an RL algorithm
for one flow field across a specified solver-budget range. Its control
realization uses a fixed-grid bridge with an exact one-step
action-reconstruction residual. It aggregates temporal probes into one
energy-based surrogate ratio and one PPO-style clip per action; the ratio is
not the executed policy's likelihood ratio. For native velocity objectives,
multi-budget endpoint losses and an anchored max--min update extend the same
deployment principle without replacing the original objective. Both
realizations retain one field and the original deployment architecture.

Our contributions are:
\begin{itemize}
\item \textbf{A diagnosis tied to deployment.} We expose one-step
degradation hidden by training-budget evaluation and identify how the
supervision target suppresses the relevant boundary signal.
\item \textbf{An algorithm for multiple solver budgets.} We connect an exact
action-reconstruction identity to action-aligned supervision, budget
exposure, and worst-budget checkpoint selection.
\item \textbf{Evidence across objectives and domains.} A controlled
common-stack benchmark and a released manipulation pipeline test the
algorithm, with component studies and deployment-cost accounting identifying
the sources of improvement. Across DM~Control, manipulation, and locomotion, \method{} leads five
flow-policy baselines in mean one-step return on \TaskCountBestOne{} of
\TaskCountEligible{} tasks including cross-regime manipulation comparisons.

\end{itemize}

\section{Problem formulation}
\label{sec:problem}

Let $s$ denote an environment state, $z\sim p_0$ source noise, and
$v_\theta(x,t\mid s)$ a velocity field with reverse-time convention
$t:1\to0$. A numerical solver $S_K$ with $K$ steps induces the policy
$a=F^{S_K}_\theta(z,s)$. The environment observes the resulting action;
the solver's internal queries do not constitute environment transitions.
Write $J_K(\theta)$ for the expected episodic return of this policy.
Our objective is to learn one field that remains useful across a specified
set of deployment budgets, without supplying $K$ as an input to the field.

Flow-policy algorithms use several tractable training signals, including
flow-matching residuals, velocity-based surrogates, and likelihoods of
stochastic internal transitions. These signals
differ in their treatment of solver time. An improvement in a training
surrogate, or in $J_{K_{\rm train}}$, does not in itself imply an improvement
in $J_K$ for a different budget. We distinguish three outcomes. \emph{Numerical failure} is a non-finite
update or evaluation. \emph{Cross-budget degradation} is a lower return at
one deployment budget than another. \emph{Temporal degradation} is a fall
in return during training at a fixed budget. The census calls a task a
\emph{one-step collapse} when its specified one-to-sixteen-step return
statistic falls below one half. This threshold is descriptive and depends
on the seed pool and aggregation rule; it is not a
synonym for temporal degradation. We report the graded return ratio
alongside the threshold count.

\section{Solver-budget mismatch}
\label{sec:census}
\label{sec:discovery}

\begin{figure}[t]
\centering
\includegraphics[width=\textwidth]{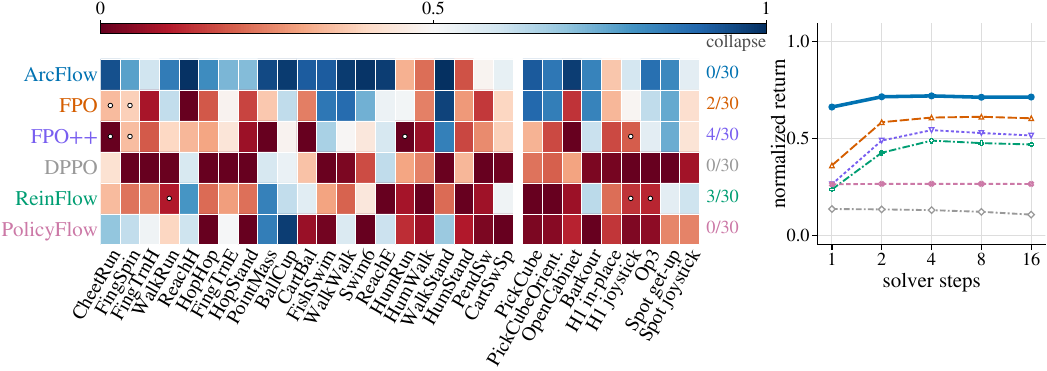}
\caption{\footnotesize \textbf{\method{} strengthens one-step deployment
across the task census and solver budgets.}
(a) Within-task normalized one-step return on all 30 retained tasks from the
registered selection results; dots mark
a one-to-sixteen-step mean-return ratio below $0.5$. The right band
shows nine Playground tasks on the display regimes of Table~\ref{tab:core}.
(b) Mean normalized return versus solver steps over all 30 tasks for every
  method: 21 held-out DM~Control tasks and 9 Playground tasks on their
  registered display regimes.}
\label{fig:census}
\end{figure}

\paragraph{Cross-budget performance.}
Figure~\ref{fig:census} compares all six flow methods on the same 30-task
census. In panel~(a), \method{} records no one-step collapses while maintaining
strong one-step returns across the DM~Control and Playground tasks. The marked
collapse counts are $0/30$ for \method{}, $2/30$ for FPO, $4/30$ for FPO++,
$0/30$ for DPPO, $3/30$ for ReinFlow, and $0/30$ for PolicyFlow. Panel~(b)
shows that \method{} has the highest mean normalized return at every evaluated
solver budget. Its aggregate return increases from one to two steps and then
remains nearly constant through sixteen steps, whereas the other methods
plateau at lower levels; the separation is largest at one-step deployment.

\paragraph{Boundary suppression.}
For a straight bridge $x_t=(1-t)a+t\varepsilon$ and target velocity
$u=\varepsilon-a$, the reference implementation predicts noise as
$\hat\varepsilon_\theta=x_t+(1-t)v_\theta(x_t,t\mid s)$. Thus
\begin{equation}
\|\hat\varepsilon_\theta-\varepsilon\|_{2,d}^2
=(1-t)^2\|v_\theta(x_t,t\mid s)-u\|_{2,d}^2,
\label{eq:eps-weight}
\end{equation}
where $\|r\|_{2,d}^2=d^{-1}\sum_i r_i^2$. Both the loss and its actor
gradient vanish at $t=1$, the only node queried by a one-step reverse Euler
solver. Nearby supervision is also quadratically down-weighted. Velocity
supervision removes this factor. The ablations in
Section~\ref{sec:mechanism} attributes most of the one-step recovery to this
change. The released ReinFlow study complements this target-level
diagnosis by testing deployment-aware training in a pipeline that already
supervises velocity (Table~\ref{tab:manipulation-summary}).

\paragraph{Collapse versus absolute performance.}
The collapse count and the mean return capture different aspects of
solver-budget behavior. DPPO and PolicyFlow have no marked collapses in
Figure~\ref{fig:census}(a), yet their mean returns remain well below
\method{} at every displayed budget in panel~(b). Conversely, a method can
achieve a higher aggregate return while still containing tasks with severe
one-to-sixteen-step degradation. Evaluating low-step return together with
within-policy retention is therefore necessary: the absence of collapse alone
does not imply an effective low-latency policy.

\section{The \method{} algorithm}
\label{sec:method}

Figure~\ref{fig:solver-bridge-overview} previews \method{}: its left panel
contrasts the budget-dependent rollout action with a reference endpoint; its
right panel connects a collected action to fresh noise on a fixed-grid bridge.
The boundary residual measures one-step reconstruction error, while temporal
probes yield one clipped surrogate weight per action. Rollout budgets cycle,
and worst-budget validation selects the checkpoint. The deployment target is
\begin{equation}
\mathcal J(\theta)=\mathbb E_{K\sim p}[J_K(\theta)].
\label{eq:arcflow-goal}
\end{equation}
Algorithm~\ref{alg:arcflow} specifies this control loop. The native-objective
image-policy variant instead uses multi-budget endpoint supervision.

\subsection{Budget-cycled rollouts}
\label{sec:draw}

For $z\sim\mathcal N(0,I)$, a reverse Euler rollout produces
\begin{equation}
x_0=z,\qquad x_{j+1}=x_j-hv_\theta(x_j,t_j\mid s),\qquad
 t_j=1-j/K,\quad h=1/K,\quad a=x_K.
\label{eq:arcflow-rollout}
\end{equation}
Training cycles through $\mathcal S=\{1,2,3,4,6,8,12,16\}$ in increasing
order, using one $K$ for every rollout in an outer iteration. This gives
equal exposure to each supported budget over a cycle. Parameters evolve across
the cycle, and the field is not conditioned on $K$.

\subsection{Boundary-aware bridge}
\label{sec:bridge}
\label{sec:boundary-exposure}

For each collected tuple $(s,a,\widehat A)$, draw fresh noise
$\widetilde z\sim\mathcal N(0,I)$ and define
\begin{equation}
\widetilde x_t=(1-t)a+t\widetilde z,\qquad u=\widetilde z-a.
\label{eq:arcflow-bridge}
\end{equation}
The bridge is an analytic reference path to the observed action, not the
trajectory that generated it. On the fixed reference grid
$G_4=\{1,3/4,1/2,1/4\}$, sample $j$ uniformly and score
\begin{equation}
\ell_j(\theta;s,a,\widetilde z)
=h\|v_\theta(\widetilde x_{t_j},t_j\mid s)-u\|_{2,d}^2,
\qquad h=1/4.
\label{eq:arcflow-energy}
\end{equation}
The mean sampled energy estimates
$L_\theta=\mathbb E_{\widetilde z,j}[\ell_j]$.
For fixed $\widetilde z$, this is $\tfrac14\sum_{j=0}^3\ell_j$,
not the unnormalized sum of grid residuals. The factor matters because it
sets the scale of the exponential weight below. Every collecting solver
uses the same grid and scale.

\begin{proposition}[Exact one-step reconstruction]
\label{prop:bridge-one-step}
For any fixed $(s,a,\widetilde z)$, the boundary residual satisfies
\[
h^{-1}\ell_0(\theta)
=\|\widetilde z-v_\theta(\widetilde z,1\mid s)-a\|_{2,d}^2
=\|F^{S_1}_\theta(\widetilde z,s)-a\|_{2,d}^2.
\]
\end{proposition}
The identity follows by substituting $t_0=1$ into
Eqs.~\ref{eq:arcflow-rollout}--\ref{eq:arcflow-energy}.
It links one quarter of the sampled-node distribution to the reconstruction
error of a one-step action. The appendix extends this connection to endpoint
error bounds under regularity assumptions
(Propositions~\ref{prop:bridge-endpoint-bound}
and~\ref{prop:bridge-cross-budget}). The resulting design separates boundary
supervision from the fixed-grid bridge used to score actions at every rollout
budget.

\paragraph{Action-level surrogate.}
Let $\bar\theta$ be the rollout parameters. Cache $M$ common noise--node
probes for the old and new policies and form
\begin{equation}
\widehat\Delta_\theta=\frac1M\sum_{m=1}^M
[\ell_{j_m}(\bar\theta)-\ell_{j_m}(\theta)],\qquad
\widehat w_\theta=\exp\{\operatorname{clip}(\widehat\Delta_\theta,-c_E,c_E)\}.
\label{eq:arcflow-weight}
\end{equation}
We maximize a PPO-style clipped surrogate \citep{schulman2017ppo},
\begin{equation}
\mathcal J_{\rm Arc}(\theta)=\mathbb E\left[
\min\left(\widehat w_\theta\widehat A,
\operatorname{clip}(\widehat w_\theta,1-\epsilon,1+\epsilon)\widehat A\right)
\right].
\label{eq:arcflow-objective}
\end{equation}
The probes are averaged before clipping, so each environment action receives
one weight and one advantage. The weight is an energy-based surrogate, not
the density ratio of the executed solver policy. Since common probes give
$\widehat\Delta_{\bar\theta}=0$, its local gradient is
\begin{equation}
\left.\nabla_\theta\mathcal J_{\rm Arc}\right|_{\bar\theta}
=-\left.\mathbb E[\widehat A\nabla_\theta\widehat L_\theta(s,a)]
\right|_{\bar\theta}.
\label{eq:arcflow-local-gradient}
\end{equation}
Positive advantages therefore decrease reconstruction energy locally.
Appendix Proposition~\ref{prop:normalized-tilt} relates this expression to
an auxiliary normalized exponential-tilt policy under its stated centering
assumptions; that auxiliary policy need not equal the deployed flow policy.

\begin{algorithm}[H]
\caption{\method{}: bridge-based control realization}
\label{alg:arcflow}
\scriptsize
\begin{pseudocode}
\item \textbf{Require:} parameters $\theta$, support $\mathcal S$, reference grid $G_4$,
validation bank $\mathcal B_{\rm val}$
\item Evaluate $\theta$ at $\mathcal K=\{1,4,16\}$ on $\mathcal B_{\rm val}$.
\item $\theta^\star\leftarrow\theta$;
$s^\star\leftarrow\min_{K\in\mathcal K}\widehat J_K(\theta;\mathcal B_{\rm val})$
\item \textbf{for} each training cycle \textbf{do}
\item \algin \textbf{for} $K\in\mathcal S=\{1,2,3,4,6,8,12,16\}$ in order \textbf{do}
\item \algin\algin Collect $K$-step rollouts and estimate advantages.
\item \algin\algin Cache $M$ noise--node probes per action and old bridge energies.
\item \algin\algin Apply Eqs.~\ref{eq:arcflow-weight}--\ref{eq:arcflow-objective}
with those probes; update the critic with the backbone rule.
\item \algin Evaluate the current policy at $\mathcal K$ on $\mathcal B_{\rm val}$.
\item \algin $s\leftarrow\min_{K\in\mathcal K}\widehat J_K(\theta;\mathcal B_{\rm val})$
\item \algin \textbf{if} $s>s^\star$ \textbf{then} $(\theta^\star,s^\star)\leftarrow(\theta,s)$
\item \textbf{return} $\theta^\star$.
\end{pseudocode}
\end{algorithm}

Algorithm~\ref{alg:arcflow} instantiates the right-hand sequence of
Figure~\ref{fig:solver-bridge-overview}: cycle $K$ for rollout collection,
score actions on fresh fixed-grid probes, and apply one clipped update per
action. After each cycle it retains the checkpoint with the highest minimum
validation return over $\mathcal K=\{1,4,16\}$.
The rollout path and reference bridge have distinct roles: the former chooses
the action that receives an environment advantage, whereas the latter scores
that action against fresh noise. Changing $K$ therefore changes which actions
enter the update without changing the grid or scale of their bridge energies.
Locally, positive advantages lower the energy of collected actions and
negative advantages raise it; the cached probes make the old and current
energies comparable during each actor update.
At every rollout budget, uniform sampling on $G_4$ assigns probability $1/4$
to $t=1$. Thus the exact one-step residual in
Proposition~\ref{prop:bridge-one-step} contributes to the expected actor update
even when rollouts take many steps. Averaging the $M$ probes before exponentiation gives
each collected action one weight, aligned with its one advantage, rather than
separate weights for individual bridge nodes.

\subsection{Training and selection}
\label{sec:certified}
\label{sec:gated-method}
\label{app:official-registration}

The reference grid $G_4$ and its $1/4$ residual scale remain fixed as
the rollout budget changes. Actor updates reuse cached probes, and the critic
follows the backbone rule. Algorithm~\ref{alg:arcflow} evaluates the initial
field and each cycle-end checkpoint at $\mathcal K=\{1,4,16\}$, then selects
\begin{equation}
\theta^\star\in\arg\max_{\theta\in\mathcal C}
\min_{K\in\mathcal K}\widehat J^{\rm val}_K(\theta),
\label{eq:checkpoint-selection}
\end{equation}
where $\mathcal C$ is the set of evaluated checkpoints.
Taking the minimum prevents a high return at one validation budget from
compensating for a low return at another. The selected $\theta^\star$ is one
shared field: deployment changes the solver budget, not the checkpoint or the
policy parameters.
Training and selection have different roles: Eq.~\ref{eq:arcflow-objective}
updates the field through the advantage-weighted bridge surrogate, whereas
Eq.~\ref{eq:checkpoint-selection} ranks the resulting checkpoints by
environment return. No gradient is taken through the validation returns.

\paragraph{Native-objective realization.}
The bridge surrogate is unnecessary when the native learner already provides
velocity supervision. In the released ReinFlow pipeline, \method{} retains
the native actor objective as an anchor, collects at $K=4$, and supervises
deterministic endpoint
views at $\mathcal V=\{1,2,4,8,16\}$ against stored rollout actions,
weighted by normalized advantage midranks within each minibatch.
Let $g_0$ be the native ascent direction and $g_K$ the negative gradient
of the endpoint loss at budget $K$; write $u_0,u_K$ for their unit directions.
The update preserves the native gradient norm while combining its direction
equally with a max--min direction across views:
\begin{equation}
 p^\star\in\arg\max_{\|p\|\le1,\ \langle p,u_0\rangle\ge0}
       \min_{K\in\mathcal V}\langle p,u_K\rangle,
 \qquad d=\|g_0\|\frac{u_0+p^\star}{\|u_0+p^\star\|}.
 \label{eq:manipulation-update-main}
\end{equation}
The native AdamW optimizer proposes a step along $d$. Backtracking accepts
a finite step only if endpoint losses stay within a small tolerance of their
pre-step values and common-noise endpoints remain close across budgets.
This check reuses training samples and makes no extra reward queries.
Appendix~\ref{app:manipulation-algorithm} gives its thresholds, clock views,
fallback, and optimizer-state handling. The separate cycling baseline changes
only the collecting solver; both realizations deploy one field.

\label{sec:not}

\begin{table}[H]
\centering
\caption{\footnotesize \textbf{Cross-domain evaluation at three deployment
budgets.} Deployed episodic return for \method{} and five flow-policy baselines
on 21 DM~Control, three manipulation, and six locomotion tasks. Bold marks the highest displayed mean for
each task and budget.}
\label{tab:core}
\vspace{1pt}
\scriptsize\setlength{\tabcolsep}{0.25pt}%
\renewcommand{\arraystretch}{0.92}
\begin{tabular*}{\textwidth}{@{\extracolsep{\fill}}l
@{\hspace{1.2pt}{\color{black!35}\vrule width 0.25pt}\hspace{1.2pt}}
rrr@{\hspace{1.2pt}{\color{black!35}\vrule width 0.25pt}\hspace{1.2pt}}
rrr@{\hspace{1.2pt}{\color{black!35}\vrule width 0.25pt}\hspace{1.2pt}}
rrr@{\hspace{1.2pt}{\color{black!35}\vrule width 0.25pt}\hspace{1.2pt}}
rrr@{\hspace{1.2pt}{\color{black!35}\vrule width 0.25pt}\hspace{1.2pt}}
rrr@{\hspace{1.2pt}{\color{black!35}\vrule width 0.25pt}\hspace{1.2pt}}
rrr@{}}
\toprule
Task
& \multicolumn{3}{c@{\hspace{1.2pt}{\color{black!35}\vrule width 0.25pt}\hspace{1.2pt}}}{FPO}
& \multicolumn{3}{c@{\hspace{1.2pt}{\color{black!35}\vrule width 0.25pt}\hspace{1.2pt}}}{FPO++}
& \multicolumn{3}{c@{\hspace{1.2pt}{\color{black!35}\vrule width 0.25pt}\hspace{1.2pt}}}{DPPO}
& \multicolumn{3}{c@{\hspace{1.2pt}{\color{black!35}\vrule width 0.25pt}\hspace{1.2pt}}}{ReinFlow}
& \multicolumn{3}{c@{\hspace{1.2pt}{\color{black!35}\vrule width 0.25pt}\hspace{1.2pt}}}{PolicyFlow}
& \multicolumn{3}{c}{\method{}} \\
\cmidrule(lr){2-4}\cmidrule(lr){5-7}\cmidrule(lr){8-10}\cmidrule(lr){11-13}\cmidrule(lr){14-16}\cmidrule(lr){17-19}
 $K$ & 1 & 4 & 16 & 1 & 4 & 16 & 1 & 4 & 16 & 1 & 4 & 16 & 1 & 4 & 16 & 1 & 4 & 16 \\
\midrule
\multicolumn{19}{c}{\textbf{DM Control }} \\
BallInCup & 738 & 856 & 856 & 656 & 852 & 841 & 683 & 647 & 607 & 735 & 850 & 850 & 935 & 937 & 935 & \textbf{937} & \textbf{949} & \textbf{949} \\
CartpoleBalance & 611 & 695 & 732 & 487 & 793 & 816 & 668 & 659 & 631 & 772 & \textbf{999} & 891 & 522 & 496 & 492 & \textbf{956} & 959 & \textbf{953} \\
CartpoleSwingupSparse & 207 & \textbf{522} & \textbf{495} & 201 & 421 & 423 & 12.1 & 12.3 & 11.2 & 274 & 286 & 263 & 4.53 & 4.94 & 4.67 & \textbf{284} & 300 & 299 \\
CheetahRun & 409 & \textbf{894} & \textbf{895} & 156 & 827 & 813 & 469 & 457 & 437 & 410 & 728 & 706 & 672 & 677 & 677 & \textbf{849} & 868 & 859 \\
FingerSpin & 299 & \textbf{670} & \textbf{666} & 276 & 588 & 571 & 54.2 & 55.0 & 51.7 & 200 & 407 & 317 & 457 & 477 & 477 & \textbf{554} & 612 & 585 \\
FingerTurnEasy & 710 & \textbf{894} & \textbf{901} & 689 & 758 & 755 & 498 & 503 & 508 & 626 & 743 & 738 & 722 & 728 & 725 & \textbf{820} & 870 & 860 \\
FingerTurnHard & 430 & 755 & \textbf{760} & 487 & 656 & 647 & 383 & 401 & 398 & 517 & 629 & 617 & 663 & 671 & 667 & \textbf{705} & \textbf{758} & 751 \\
FishSwim & 575 & \textbf{638} & \textbf{635} & 485 & 588 & 580 & 187 & 189 & 185 & 320 & 500 & 497 & 193 & 195 & 196 & \textbf{600} & 620 & 616 \\
HopperHop & 14.4 & 32.2 & 30.6 & 22.8 & 36.6 & 33.0 & 0.30 & 0.30 & 0.30 & 16.4 & 44.0 & 44.0 & 0.44 & 0.41 & 0.42 & \textbf{60.4} & \textbf{74.1} & \textbf{73.2} \\
HopperStand & 85.2 & 108 & 110 & 39.4 & 75.1 & 70.9 & 4.92 & 4.89 & 4.74 & 115 & 243 & 234 & 4.28 & 3.89 & 3.90 & \textbf{361} & \textbf{503} & \textbf{496} \\
HumanoidRun & \textbf{6.98} & \textbf{7.58} & \textbf{7.29} & 1.51 & 5.77 & 4.98 & 4.73 & 4.74 & 4.59 & 2.42 & 3.22 & 2.73 & 2.90 & 2.92 & 2.90 & 4.94 & 5.66 & 5.41 \\
HumanoidStand & 23.8 & 26.0 & 24.3 & 17.7 & \textbf{30.8} & \textbf{24.8} & \textbf{26.5} & 25.7 & 23.8 & 9.06 & 11.4 & 9.66 & 14.2 & 14.2 & 14.2 & 18.5 & 23.0 & 21.5 \\
HumanoidWalk & \textbf{12.2} & \textbf{12.9} & \textbf{13.0} & 7.19 & 12.5 & 9.91 & 12.1 & 12.0 & 11.6 & 5.35 & 6.06 & 5.98 & 7.31 & 8.07 & 8.13 & 11.3 & 12.3 & 11.7 \\
PendulumSwingup & 35.9 & 56.7 & 56.2 & 63.9 & 66.8 & 63.9 & 0.98 & 1.09 & 1.02 & 19.5 & 30.6 & 28.3 & 8.34 & 7.65 & 7.77 & \textbf{116} & \textbf{142} & \textbf{141} \\
PointMass & 715 & 840 & 836 & 610 & 830 & 842 & 777 & 771 & 756 & 849 & 869 & 869 & 851 & 853 & 852 & \textbf{881} & \textbf{885} & \textbf{883} \\
ReacherEasy & 921 & 968 & 968 & 927 & 966 & 963 & 931 & 930 & 923 & 866 & 964 & 964 & 929 & 935 & 934 & \textbf{968} & \textbf{970} & \textbf{969} \\
ReacherHard & 675 & \textbf{958} & \textbf{956} & 770 & 956 & 954 & 840 & 834 & 808 & 899 & 945 & 942 & 846 & 849 & 849 & \textbf{956} & 955 & 954 \\
Swimmer6 & 508 & 545 & 544 & 409 & 444 & 434 & 319 & 321 & 314 & 409 & 478 & 482 & 261 & 261 & 261 & \textbf{591} & \textbf{595} & \textbf{591} \\
WalkerRun & 416 & \textbf{620} & \textbf{614} & 292 & 484 & 475 & 80.6 & 78.8 & 77.1 & 130 & 312 & 308 & 290 & 289 & 289 & \textbf{537} & 567 & 563 \\
WalkerStand & 949 & 965 & 964 & 871 & 939 & 938 & 344 & 343 & 340 & 477 & 638 & 638 & 773 & 779 & 781 & \textbf{970} & \textbf{972} & \textbf{971} \\
WalkerWalk & 877 & \textbf{955} & \textbf{956} & 606 & 839 & 825 & 286 & 274 & 265 & 408 & 734 & 721 & 660 & 662 & 663 & \textbf{942} & 950 & 949 \\
\midrule
\multicolumn{19}{c}{\textbf{Manipulation }} \\
PandaPickCube & 700 & 750 & 748 & 332 & 392 & 402 & 267 & 265 & 258 & 119 & 130 & 132 & 138 & 151 & 158 & \textbf{719} & \textbf{769} & \textbf{754} \\
PandaPickCubeOrient & 803 & 863 & 860 & 250 & 295 & 303 & 273 & 271 & 264 & 109 & 115 & 115 & 145 & 165 & 179 & \textbf{817} & \textbf{931} & \textbf{919} \\
PandaOpenCabinet & 131 & 164 & 167 & 50.7 & 58.0 & 57.8 & 225 & 224 & 220 & 92.2 & 94.9 & 93.8 & 118 & 131 & 133 & \textbf{618} & \textbf{632} & \textbf{617} \\
\midrule
\multicolumn{19}{c}{\textbf{Locomotion }} \\
BarkourJoystick & \textbf{6.06} & 6.59 & 6.50 & 4.58 & \textbf{7.07} & 6.83 & 0.72 & 0.70 & 0.61 & 4.72 & 7.02 & 6.96 & 0.53 & 0.54 & 0.56 & 6.04 & 6.22 & \textbf{7.15} \\
H1InplaceGait & 13.5 & 15.5 & 15.5 & 11.1 & 18.0 & 17.9 & 8.84 & 8.77 & 8.52 & 11.8 & \textbf{23.2} & \textbf{23.4} & 10.6 & 10.8 & 10.9 & \textbf{13.9} & 14.6 & 19.6 \\
H1JoystickGait & 2.30 & 3.28 & 3.30 & 0.96 & 2.11 & 2.10 & 0.06 & 0.06 & 0.06 & 0.69 & \textbf{4.61} & \textbf{4.54} & 0.21 & 0.22 & 0.22 & \textbf{2.73} & 3.48 & 4.22 \\
Op3Joystick & 7.54 & 10.8 & 11.0 & 6.68 & 11.1 & 10.6 & 0.20 & 0.20 & 0.17 & 1.58 & 11.8 & 11.8 & 0.79 & 0.85 & 0.89 & \textbf{10.5} & \textbf{11.9} & \textbf{11.9} \\
SpotGetup & 7.23 & 7.55 & 7.53 & 7.16 & \textbf{8.83} & \textbf{8.81} & 2.14 & 2.13 & 2.12 & 5.86 & 8.25 & 8.22 & 3.86 & 4.58 & 4.66 & \textbf{7.66} & 8.78 & 8.65 \\
SpotJoystickGait & 22.9 & 28.0 & 28.0 & 23.2 & 29.7 & 29.3 & 16.5 & 16.4 & 16.1 & \textbf{26.5} & \textbf{34.0} & \textbf{33.9} & 19.8 & 15.8 & 15.0 & 25.5 & 29.5 & 29.2 \\
\bottomrule
\end{tabular*}
\end{table}

\section{Experiments}
\label{sec:experiments}

\subsection{Experimental setup}
\label{sec:setup}

\paragraph{Common-stack benchmark.}
We instantiate all methods in a control stack derived from the released FPO
Playground configuration. It covers \ScaleTasks{} DM~Control tasks
\citep{tassa2018dmcontrol}, \ScaleOuterIterations{} outer iterations, and
\ScaleInteractions{} interactions; fixed-budget arms collect at
\ScaleTrainNFE{} steps. We compare FPO \citep{mcallister2025fpo}, FPO++
\citep{yi2026fpopp}, DPPO \citep{ren2025dppo}, ReinFlow
\citep{zhang2025reinflow}, and PolicyFlow \citep{yang2026policyflow}. All arms
share the normalizer, architecture, initialization, critic, and training-interaction
budget. The DPPO row applies its denoising-MDP policy-gradient objective to
the shared flow actor, not the original diffusion actor; Appendix
\ref{app:baseline-audit} specifies the adaptation. Raw returns use each configuration's registered seeds and display
regime; the Panda comparisons are cross-regime. Figure~\ref{fig:census}
min--max normalizes each task over six methods, five budgets, and a Gaussian
reference. ``Collapse'' means $R_1/R_{16}<0.5$. The separate \texttt{can}
study uses five matched training seeds.

\subsection{Multi-budget performance}
\label{sec:census-results}
\label{sec:core}

\paragraph{Broad low-budget gains.}
Table~\ref{tab:core} reports the deployed return of all six methods at three
deployment budgets on 30 tasks. At $K=1$, \method{} has the highest displayed
mean on \TaskCountBestOne{} of \TaskCountEligible{} tasks under the registered
display regimes: \TaskCountDMOne{} of 21 DM~Control tasks, all
\TaskCountPandaOne{} manipulation tasks, and \TaskCountLocoOne{} of six
locomotion tasks. Across the 25 wins, the strongest baseline is FPO on 11
tasks, PolicyFlow on six, ReinFlow on four, and FPO++ and DPPO on two each. The gains are also often substantial. Using the displayed means, 12 of the 25
wins improve on the strongest baseline by at least $10\%$, and nine improve by
at least $20\%$. The largest multiplicative gains occur on HopperStand, where
\method{} reaches $361$ versus $115$ for the strongest baseline; HopperHop,
with $60.4$ versus $22.8$; and PandaOpenCabinet, with $618$ versus $225$. The
advantage is not limited to these outliers: at one step, \method{} also obtains
$956$ versus $772$ on CartpoleBalance, $849$ versus $672$ on CheetahRun, $554$
versus $457$ on FingerSpin, and $10.5$ versus $7.54$ on Op3Joystick. Thus the
lowest-cost setting yields both broad task coverage and pronounced gains on
several of the more difficult control and manipulation problems.

\begin{figure}[H]
\centering
\includegraphics[width=\textwidth]{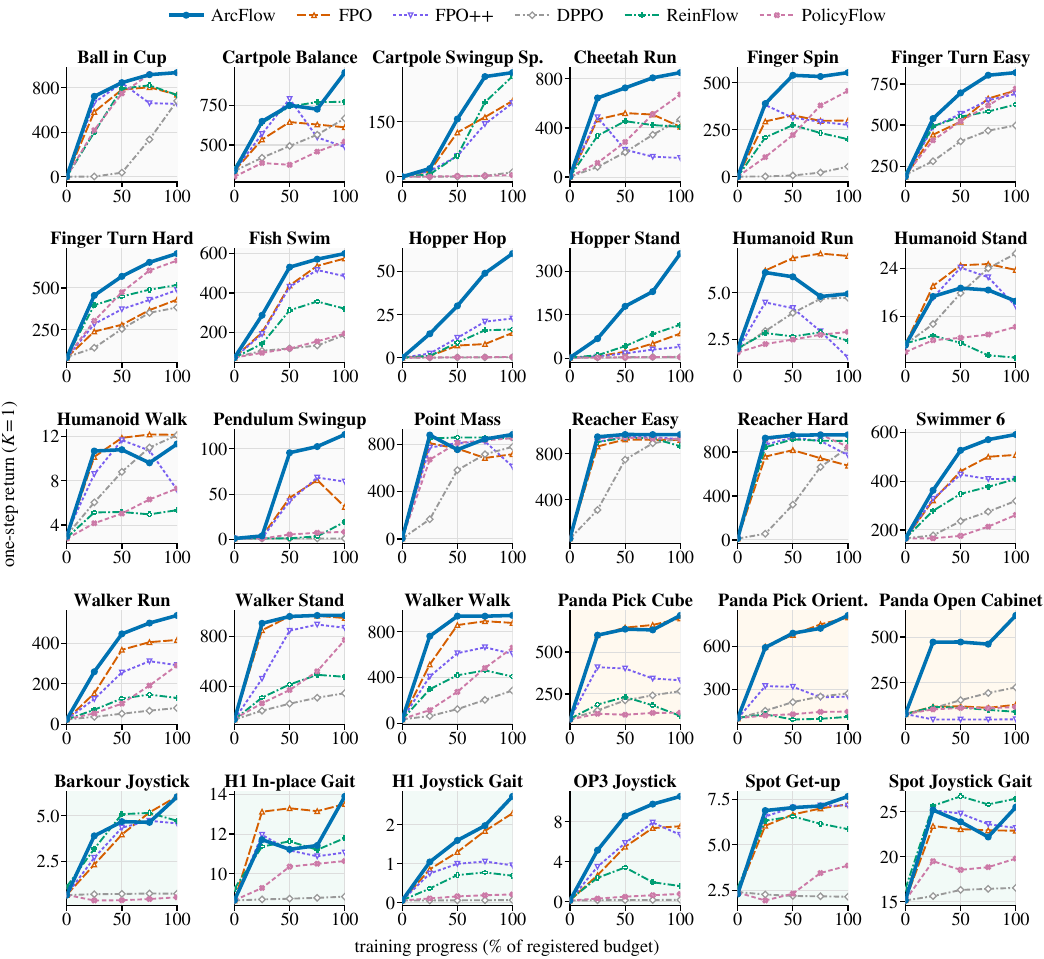}
  \caption{\footnotesize
  \textbf{One-step learning dynamics across 30 tasks.}
  Mean raw episodic return at $K=1$, ordered as in Table~\ref{tab:core}.
  Training progress is the percentage of each method's registered interaction
  budget. Curves use the
  registered seeds. Axis ranges are task-specific; compare within panels.}
  
\label{fig:degradation}
\end{figure}

\paragraph{Domain and budget coverage.}
The manipulation block provides the clearest uniform result: \method{} has the
highest displayed mean at all three budgets on PandaPickCube,
PandaPickCubeOrient, and PandaOpenCabinet. In DM~Control, it is best at all
three budgets on eight tasks, including BallInCup, HopperHop, HopperStand,
PendulumSwingup, PointMass, ReacherEasy, Swimmer6, and WalkerStand; it also
leads all three budgets on Op3Joystick in locomotion. Across the full table,
\method{} is best on 13 tasks at $K=4$ and 14 at $K=16$, and on 12 tasks at
every reported budget. Task-level leaders are therefore more mixed at higher
solver counts than at one step, but the table shows that \method{} combines a
broad low-budget advantage with strong performance as additional computation
becomes available. These task-level counts identify where each method leads;
Figure~\ref{fig:census}(b) instead shows how each trained field performs as
the deployment budget changes, averaged over the 30 tasks. \method{} has the
highest mean normalized return at every evaluated solver count, with the
largest separation at one step. Its mean rises from one to two steps and then
remains nearly constant through sixteen steps. Figure~\ref{fig:census}(a)
likewise records no one-step collapse for \method{} across the 30-task census.
Although the strongest high-budget method varies on the remaining tasks, the
aggregate curve stays above all baselines. Thus the one-step gains do not come
at the cost of losing the policy's higher-budget capability.

\begin{table}[H]
\begin{minipage}[t]{0.49\textwidth}
\centering
\refstepcounter{table}\label{tab:mechanism-main}
\parbox{\linewidth}{\raggedright\footnotesize Table~\thetable:
\textbf{Boundary supervision, bridge training, and worst-$K$ selection improve
deployment robustness.} Point-estimate means on a 25-task component pool.
$K=1$ and Worst $K$ evaluate normalized deployment return, with Worst $K$
summarizing the weakest budget over $\mathcal K$. AUC is the mean four-step training-curve area, min--max normalized
within each task over these rows. Bold: column maximum.}
\vspace{4pt}
\footnotesize\setlength{\tabcolsep}{1.2pt}%
\renewcommand{\arraystretch}{1.13}
\begin{tabular}{@{}lccc@{}}
\toprule
Configuration & $K=1$ & Worst $K$ & AUC \\
\midrule
FPO & $0.350$ & $0.334$ & $0.587$ \\
$+$ $u$-supervision & $0.534$ & $0.508$ & $0.333$ \\
bridge, fixed $K=4$ & $0.612$ & $0.593$ & $\mathbf{0.700}$ \\
bridge + cycle $\{1,4,16\}$ & $0.646$ & $0.635$ & $0.612$ \\
bridge + cycle $\mathcal S$ & $0.668$ & $0.623$ & $0.661$ \\
\textbf{\method{} (+ worst-$K$)} & $\mathbf{0.711}$ & $\mathbf{0.685}$ & $0.667$ \\
\bottomrule
\end{tabular}

\end{minipage}\hfill
\begin{minipage}[t]{0.48\textwidth}
\centering
\refstepcounter{table}\label{tab:manipulation-summary}
\parbox{\linewidth}{\raggedright\footnotesize Table~\thetable:
\textbf{\method{} restores full-range manipulation.} Mean \texttt{can} success
(\%) over five matched training seeds. BC is the behavior-cloned prior;
``1-step'' is the specialist. Worst and mean $K$ summarize the five budget
means. Bold: row maximum.}
\vspace{4pt}
\footnotesize\setlength{\tabcolsep}{1.2pt}
\begin{tabular}{@{}lccccc@{}}
\toprule
Budget & BC & ReinFlow & 1-step & Cycling & \method{} \\
\midrule
$K=1$ & \TableThreeBCOne & \TableThreeReinFlowOne & {\boldmath\TableThreeOneStepOne} & \TableThreeCyclingOne & \TableThreeArcFlowOne \\
$K=2$ & \TableThreeBCTwo & \TableThreeReinFlowTwo & \TableThreeOneStepTwo & {\boldmath\TableThreeCyclingTwo} & \TableThreeArcFlowTwo \\
$K=4$ & \TableThreeBCFour & \TableThreeReinFlowFour & \TableThreeOneStepFour & {\boldmath\TableThreeCyclingFour} & \TableThreeArcFlowFour \\
$K=8$ & \TableThreeBCEight & \TableThreeReinFlowEight & \TableThreeOneStepEight & \TableThreeCyclingEight & {\boldmath\TableThreeArcFlowEight} \\
$K=16$ & \TableThreeBCSixteen & \TableThreeReinFlowSixteen & \TableThreeOneStepSixteen & \TableThreeCyclingSixteen & {\boldmath\TableThreeArcFlowSixteen} \\
\midrule
Worst $K$ & \TableThreeBCWorst & \TableThreeReinFlowWorst & \TableThreeOneStepWorst & \TableThreeCyclingWorst & {\boldmath\TableThreeArcFlowWorst} \\
Mean $K$ & \TableThreeBCMean & \TableThreeReinFlowMean & \TableThreeOneStepMean & \TableThreeCyclingMean & {\boldmath\TableThreeArcFlowMean} \\
\bottomrule
\end{tabular}

\end{minipage}
\end{table}

\subsection{Training dynamics}

\paragraph{Early emergence of low-step gains.}
Figure~\ref{fig:degradation} shows that the deployed advantage often emerges
before training ends. On CheetahRun, FingerSpin, HopperHop, HopperStand, the
Walker tasks, and the manipulation tasks, \method{} separates from most
baselines by roughly mid-training. Its Hopper returns rise steadily while
competing curves remain near the floor or plateau lower, so Table~\ref{tab:core}'s
differences are not created solely by the deployed endpoints.

\paragraph{Stable and late-stage improvement.}
On BallInCup, PointMass, ReacherEasy, and ReacherHard, \method{} rises early and
remains stable; PendulumSwingup, Swimmer6, H1JoystickGait, and Op3Joystick keep
improving after mid-training. On CartpoleBalance, PandaOpenCabinet, and
H1InplaceGait, the selected rightmost marker rises above the preceding
chronological trajectory. The deployed result thus combines early acquisition,
continued gains, and checkpoint selection.

\subsection{Control component analysis}
\label{sec:mechanism}

Table~\ref{tab:mechanism-main} isolates the contribution of each component on
the separate 25-task component-study pool. Replacing noise reconstruction with direct
$u$-supervision raises normalized one-step return from $0.350$ to $0.534$ and
worst-budget return from $0.334$ to $0.508$, confirming that restoring the
boundary signal accounts for a substantial part of the recovery. This direct
repair alone reduces training AUC to $0.333$. Introducing the fixed-scale
bridge at $K=4$ raises one-step and worst-budget return further to $0.612$ and
$0.593$, while producing the highest training AUC, $0.700$. This row shows the
combined benefit of the fixed-grid bridge formulation and its fixed scale.

Budget exposure provides the next improvement. Cycling over
$\{1,4,16\}$ raises the worst-budget return to $0.635$, and the broader support
$\mathcal S$ reaches $0.668$ at one step with training AUC $0.661$.
Worst-budget checkpoint selection completes the algorithm: the full
\method{} row reaches $0.711$ at one step and $0.685$ at its weakest evaluated
budget, the strongest deployment values in the table. The ladder shows that
the components play distinct roles: boundary supervision repairs low-step
execution, the fixed-grid bridge improves the training curve, cycling supplies
budget coverage, and selection converts that coverage into a stronger
deployment floor.

\subsection{Native-pipeline validation}
\label{sec:testbed}

Table~\ref{tab:manipulation-summary} tests whether the same deployment-budget
principle transfers beyond the reconstruction-based control realization. The
released ReinFlow ShortCut pipeline starts from a behavior-cloned image policy,
uses native velocity supervision, and fine-tunes a four-step actor. It tests the
principle beyond FPO's reconstruction target and compares one multi-budget
policy with a one-step specialist and rollout-budget cycling. The released behavior-cloned prior retains moderate performance across solver
budgets, but official ReinFlow fine-tuning raises four-step success to
\TableThreeReinFlowFour{}\% while reducing one-step success to
\TableThreeReinFlowOne{}\%. The specialist reverses this tradeoff: it reaches
\TableThreeOneStepOne{}\% at one step but has a worst-budget score of only
\TableThreeOneStepWorst{}\%. Cycling performs best at two and four steps, yet
falls to \TableThreeCyclingWorst{}\% at its weakest budget. In contrast,
\method{} remains between \TableThreeArcFlowWorst{}\% and
\TableThreeArcFlowSixteen{}\% across all five budgets, giving the highest
worst-budget score (\TableThreeArcFlowWorst{}\%) and mean score
(\TableThreeArcFlowMean{}\%). Thus, its advantage is the
high performance floor obtained with one field across the deployment range.

\section{Related work}
\label{sec:related}

\textbf{Generative transport and control.}
Diffusion models learn reverse transport from noise to data
\citep{sohldickstein2015deep,ho2020ddpm,song2021sde}, while deterministic
samplers and neural ODEs expose an accuracy--computation tradeoff
\citep{song2021ddim,chen2018neuralode}. Flow matching and rectified flow learn
velocity fields directly \citep{lipman2023flow,liu2023rectified}; transport
couplings and multisample matching straighten paths, and rectified flows scale
to high-resolution generation
\citep{tong2024otcfm,pooladian2023multisample,esser2024sd3}. In control,
generative models support planning and return conditioning
\citep{janner2022diffuser,ajay2023decision}, plus visuomotor imitation across
2D, goal-conditioned, and 3D inputs
\citep{chi2023diffusionpolicy,reuss2023beso,ze2024dp3}. AdaFlow adapts
integration to action variance and FlowPolicy learns one-step consistency
\citep{hu2024adaflow,zhang2025flowpolicy}. This work instead trains one
flow field with RL for a declared set of deployment budgets.

\textbf{Reinforcement learning of generative policies.}
Trust-region and clipped policy objectives motivate many modern policy updates
\citep{schulman2015trpo,schulman2017ppo}. DDPO casts denoising as an MDP, while
DRaFT differentiates rewards through the sampler
\citep{black2024ddpo,clark2024draft}. For control, DPPO applies policy gradients
to diffusion-policy denoising transitions and QVPO uses a value-weighted update
\citep{ren2025dppo,ding2024qvpo}. ReinFlow makes flow transitions stochastic;
FPO and FPO++ form PPO-style ratios from flow losses; PolicyFlow and Flow-GRPO
derive velocity- or transition-based objectives
\citep{zhang2025reinflow,mcallister2025fpo,yi2026fpopp,yang2026policyflow,liu2025flowgrpo}.
Offline and critic-based methods include Diffusion-QL, EDP, DTQL, BDPO,
Q-score matching, Flow Q-Learning, and SAC Flow
\citep{wang2023diffusionql,kang2023edp,chen2024dtql,gao2025bdpo,psenka2024qscore,park2025fql,zhang2026sacflow}.
\method{} focuses on the return floor of one policy across solver budgets.

\textbf{Few-step and one-step policies.}
DDIM and DPM-Solver accelerate sampling numerically
\citep{song2021ddim,lu2022dpmsolver}; progressive distillation, consistency
models, and InstaFlow learn fast generators
\citep{salimans2022progressive,song2023consistency,liu2024instaflow}. Shortcut
models and MeanFlow train directly over variable or one-step generation
\citep{frans2025shortcut,geng2025meanflow}, and Consistency Policy distills a
visuomotor policy \citep{prasad2024consistencypolicy}. One-step flow policy
mirror descent, FLAME, and MFPO bring this objective to RL
\citep{chen2025fpmd,li2026flame,dong2026mfpo}. These methods optimize a
designated short sampler or fast policy; \method{} preserves one RL-trained
field and measures it at every declared $K\in\mathcal K$.

\textbf{Training-time weighting and deployment budgets.}
P2, Min-SNR, and EDM reweight diffusion times
\citep{choi2022p2,hang2023minsnr,karras2022edm}. Our problematic weight instead
follows algebraically from noise reconstruction and is zero at the one-step
boundary. Boundary supervision restores that signal, solver cycling exposes
the intended budgets, and worst-budget selection makes their return floor an
explicit selection criterion.

\section{Conclusion}
\label{sec:discussion}
\method{} aligns flow-policy learning with its deployment range through
boundary supervision, solver-budget exposure, and worst-budget selection.
Across three control domains it improves one-step performance while retaining
competitive higher-budget returns; the native ReinFlow study extends the
principle to a released image-policy pipeline and five budgets. Deployment
retains one architecture and pays only for the selected solver budget.
Component studies further separate the gains from action-aligned supervision,
budget exposure, and selection. These results support training and evaluating
generative policies over their full deployment range.

\FloatBarrier
\label{main-end}

\clearpage
\phantomsection
\label{statement-start}
\subsection*{AI use statement}
Generative AI tools were used to help develop and critique the conceptual
framework and hypotheses; formulate mathematical claims and proofs; design
synthetic environments and experimental protocols; implement, test, and edit
software; interpret results; search literature; prepare figures and tables;
and draft and edit the manuscript. AI-generated outputs were not treated as
empirical evidence: reported values come from versioned artifacts, citations
are checked against primary sources, and assumptions and proofs are stated in
the paper. Generative AI was not used to generate or clean experimental
datasets or to translate source material; qualitative or thematic data
analysis was not applicable. We reviewed the AI-assisted work and take
responsibility for the final text, claims, code, and artifacts.

\subsection*{Ethics statement}
This work studies reinforcement learning for simulated control and does not
involve human participants, personal data, or physical-robot experiments.
Like other reinforcement-learning methods, \method{} may amplify errors or
unintended incentives in a specified reward function; improved optimization
does not by itself provide safety guarantees. Applying the method to physical
or safety-critical systems would therefore require independently validated
objectives, operational constraints, monitoring, and fail-safe mechanisms.
We report task-level outcomes and the limits of the present evaluation to
support accurate interpretation of the results.

\subsection*{Reproducibility statement}
Code and artifacts are available in the anonymous repository at
\url{https://anonymous.4open.science/r/arcflow-868A}. The repository contains
the \method{} implementation, experiment configurations.
The main text and appendices specify the algorithms, solver budgets,
checkpoint-selection rule, evaluation conventions, and matched five-seed
manipulation protocol; the repository provides the experiment configurations.

\bibliography{references}
\bibliographystyle{iclr2027_conference}

\appendix
\clearpage
\FloatBarrier
\section{Limitations and scope}
\label{app:limitations}

\paragraph{Empirical scope.}
Our evidence is limited to simulated control. The common-stack study covers
DM~Control, simulated Panda manipulation, and locomotion tasks, while the
native image-policy study examines one manipulation task. These experiments do
not establish robustness to real-robot sensing and actuation, safety
constraints, distribution shift, or substantially larger policy and action
spaces. The native study also implements the deployment-budget principle with
endpoint supervision rather than the bridge objective, so it supports transfer
of the principle but is not an independent test of the same update rule.

\paragraph{Comparison and statistical scope.}
Most baselines in the common-stack benchmark are controlled adaptations to a
shared flow backbone, critic, and interaction budget; they are not
reproductions of each method's native architecture, pretraining, or published benchmark.

\paragraph{Training and selection cost.}
\method{} leaves deployment architecture and per-step solver cost unchanged,
but training is not free: the bridge realization evaluates cached probes,
cycles through several rollout budgets, and repeatedly measures validation
return at $K\in\{1,4,16\}$. In the reported control protocol, validation adds
$2.75\%$ to the environment-interaction count, in addition to neural-network
compute from the probes and multi-budget evaluations. Moreover, selection protects only the budgets it evaluates and gives no guarantee for unseen solver counts.

\paragraph{Objective and theory.}
The bridge weight is an energy-based surrogate, not the likelihood ratio of the
executed policy, and PPO-style clipping therefore has no exact trust-region
interpretation. The appendix bounds finite-solver reconstruction error for a
deterministic Euler field under Lipschitz assumptions; it does not prove
monotone improvement in environment return, and it does not directly cover the
clipped ShortCut operator used in the native pipeline. The fixed four-node
bridge, the chosen budget sets, and the associated thresholds are empirical
design choices. Whether the same choices remain effective with adaptive or
higher-order solvers, continuous budget ranges, longer horizons, or other
generative-policy parameterizations remains open.

\section{Implementation details}
\label{app:details}

Algorithm~\ref{alg:arcflow} and the complete bridge-based control realization
appear in Section~\ref{sec:certified}. This supplement gives the mathematical
details, baseline scope, energy-scale diagnostic, and exact native-anchored
manipulation update.

\subsection{Native-objective manipulation update}
\label{app:manipulation-algorithm}
\label{app:testbed-graft}

The configuration collects with the released
ReinFlow ShortCut actor at four steps. Its extra supervision evaluates
the same actor at $\mathcal V=\{1,2,4,8,16\}$ using the stored rollout
noise and actions. ShortCut retains its native time and duration inputs;
the reverse-time field in the bridge realization has only time. The following
specifies
the actual manipulation update, including its distinction from the bridge
loss of Eq.~\ref{eq:arcflow-objective}.

\paragraph{Endpoint views and advantage weights.}
For each captured minibatch $B$ of $(s_i,z_i,a_i,\widehat A_i)$,
$z_i$ and $a_i$ are the initial and final stored chain states. Let
$E_{K,\alpha}(\theta;z,s)$ be the released ShortCut field integrated on
$t_j=(j/K)^\alpha$, using its native forward-time convention, duration
$t_{j+1}-t_j$, and action clipping after every step. The trained views
use $\alpha=1$. Within each minibatch, define normalized empirical
midrank weights
\begin{equation}
 r_i=\sum_{j\in B}\mathbf1\{\widehat A_j<\widehat A_i\}
 +\tfrac12\left(1+\sum_{j\in B}\mathbf1\{\widehat A_j=\widehat A_i\}\right),
 \qquad w_i=r_i\Big/\sum_{j\in B}r_j.
\end{equation}
The endpoint objective is
\begin{equation}
 L_K^{\rm end}(\theta)=\sum_B\sum_{i\in B}w_i
 \|E_{K,1}(\theta;z_i,s_i)-\operatorname{stopgrad}(a_i)\|_{2,d}^2.
 \label{eq:manipulation-endpoint-loss}
\end{equation}
The loss sums the normalized minibatch losses. All $15$ accumulation
minibatches on \texttt{can} contribute to
every view gradient. This matches the native update's accumulation
window. Rank weights favor high-advantage actions and handle ties with
their average rank; they are not signed policy-gradient coefficients.

\paragraph{Native anchor and direction.}
Let $g_0$ be the negative of the accumulated native actor gradient
(policy loss plus the release's entropy and behavior-cloning terms), and
$g_K=-\nabla L_K^{\rm end}$. Normalize these to unit vectors $u_0,u_K$.
The small view-space solve is
\begin{equation}
 p^\star\in\arg\max_{\|p\|\le1,\ \langle p,u_0\rangle\ge0}
             \min_{K\in\mathcal V}\langle p,u_K\rangle.
 \label{eq:manipulation-maxmin}
\end{equation}
It is implemented through the simplex-plus-ray dual
$\min_{\beta\in\Delta,\lambda\ge0}\|\sum_K\beta_Ku_K+\lambda u_0\|$,
using a float64 Gram matrix and active-set enumeration. The applied
direction uses a fixed equal blend,
\begin{equation}
 d=\|g_0\|\frac{\tfrac12u_0+\tfrac12p^\star}
                      {\|\tfrac12u_0+\tfrac12p^\star\|}.
 \label{eq:manipulation-anchor}
\end{equation}
A degenerate max--min solution falls back to $d=g_0$; zero or non-finite
input gradients are rejected by the implementation. The norm floor is
$10^{-12}$, feasibility tolerance $10^{-8}$, and cone-apex tolerance
the square root of the input dtype's machine epsilon. The equal blend
is an explicit native anchor, applied after the max--min solve.

\paragraph{Step acceptance.}
Replace the actor's gradient by $-d$ and call the native AdamW step to
obtain a candidate parameter state $\theta_c$ and optimizer state.
Try $\theta_f=\theta_0+f(\theta_c-\theta_0)$ in descending order
$f\in\{1,1/2,1/4,1/8,0\}$. Accept the first finite candidate for which
every $L_K^{\rm end}$ increases by at most $10^{-7}$ from its pre-step value
on the captured minibatches and the maximum old-to-candidate common-noise
endpoint RMSE is at most
$0.25+10^{-7}$. The displacement check uses the first $32$ stored
noise--state pairs and all $15$ views
$\mathcal V\times\{\alpha=1,2,4\}$. The view losses use only the
native-clock views; the clock-warped views constrain displacement.

Actor optimizer moments interpolate between the old and candidate
moments with the accepted fraction $f$. A positive fraction retains
the candidate step counter; $f=0$ restores the old actor state completely.
The critic's native update is unchanged. Acceptance reuses training
advantages and samples and makes no additional reward queries; it is a
finite-step surrogate check, not an independent return certificate.
The separate cycling control uses the native update and cycles
$K\in\{1,2,4\}$ at iteration boundaries, with fixed-$K$ validation.

\FloatBarrier
\section{Mathematical basis}
\label{app:theory}

The statements below concern the deterministic reverse Euler field in
Section~\ref{sec:method}. They identify the supervised deployment boundary
and bound action reconstruction under explicit regularity assumptions.
They do not assert monotone environment return or transfer directly to
the clipped ShortCut operator in Appendix~\ref{app:manipulation-algorithm}.

\subsection{Boundary non-coercivity}

\begin{proposition}[Boundary non-coercivity]
\label{prop:bridge-noncoercive-app}
Fix a scalar velocity error $c$ and let
$b_\delta(t)=c\max\{0,(t-(1-\delta))/\delta\}$. This perturbation is continuous,
has one-step boundary error $b_\delta(1)^2=c^2$, and has continuous
noise-reconstruction risk
$\int_0^1(1-t)^2b_\delta(t)^2\,dt=c^2\delta^3/30$. Thus arbitrarily small
continuous surrogate risk does not control the deployed one-step endpoint error,
even over continuous fields.
\end{proposition}
\emph{Proof.} Substitute $q=(t-(1-\delta))/\delta$ on $[1-\delta,1]$:
$\int_0^1(1-t)^2b_\delta(t)^2\,dt
=c^2\delta^3\int_0^1q^2(1-q)^2\,dq=c^2\delta^3/30$,
while $b_\delta(1)^2=c^2$. \qed

\subsection{Finite-solver endpoint bound}

\begin{proposition}[Finite-solver endpoint bound]
\label{prop:bridge-endpoint-bound}
With $x^\star_j$ the bridge nodes, $\widehat x_j$ the Euler trajectory from
the same source, $r_j=v_\theta(x^\star_j,t_j\mid s)-u$, and $v_\theta$
$L$-Lipschitz in its state argument uniformly over the grid times,
\begin{equation}
 \|\widehat x_K-a\|
 \le h\sum_{j=0}^{K-1}(1+hL)^{K-1-j}\|r_j\|,
 \quad
 \|\widehat x_K-a\|^2
 \le h^2\!\left[\sum_j(1+hL)^{2(K-1-j)}\right]
          \left[\sum_j\|r_j\|^2\right].
 \label{eq:bridge-endpoint-bound}
\end{equation}
\end{proposition}

Let $e_j=\widehat x_j-x_j^\star$, $h=1/K$, and
$r_j=v_\theta(x_j^\star,t_j\mid s)-u$. Lipschitz continuity gives
\begin{equation}
 \|e_{j+1}\|\le(1+hL)\|e_j\|+h\|r_j\|.
\end{equation}
With $e_0=0$, induction yields
\begin{equation}
 \|e_K\|\le h\sum_{j=0}^{K-1}(1+hL)^{K-1-j}\|r_j\|.
\end{equation}
Cauchy--Schwarz gives the squared form. \qed

\subsection{One-step identity and local gradient}

On the fixed reference grid $G_4$, $h=1/4$, $t_0=1$, and the boundary query is
$\widetilde x_{t_0}=\widetilde z$. Substituting $u=\widetilde z-a$ gives
\begin{align}
 h^{-1}\ell_0(\theta)
 &=\|v_\theta(\widetilde z,1\mid s)-(\widetilde z-a)\|_{2,d}^2\\
 &=\|\widetilde z-v_\theta(\widetilde z,1\mid s)-a\|_{2,d}^2.
\end{align}
The deployed one-step reverse-Euler endpoint from the same source is
$F^{\mathrm{Euler}_1}_\theta(\widetilde z,s)=
\widetilde z-v_\theta(\widetilde z,1\mid s)$, proving
Proposition~\ref{prop:bridge-one-step}.

Write $\widehat L_\theta=M^{-1}\sum_m\ell_m(\theta)$,
$\widehat\Delta_\theta=\widehat L_{\bar\theta}-\widehat L_\theta$, and
$\widehat w_\theta=\exp\{\operatorname{clip}(\widehat\Delta_\theta,-c_E,c_E)\}$.
At $\theta=\bar\theta$, $\widehat\Delta=0$, $\widehat w=1$, and, because the
inner cap is locally inactive,
$\nabla\widehat w=-\nabla\widehat L$. Away from the outer clip kink, the actor
objective therefore has derivative $-\widehat A\nabla\widehat L$. Taking
expectation establishes Eq.~\ref{eq:arcflow-local-gradient}. This derivation
requires common probes: independent old and new probes would make the sampled
energy difference nonzero even at identical parameters.

\begin{proposition}[Normalized local-policy interpretation]
\label{prop:normalized-tilt}
Let $L_\theta(s,a)=\mathbb E_{\widetilde z,j}[\ell_j(\theta)]$ and define
the normalized local policy
\begin{equation}
 q_\theta(a\mid s)=
 \frac{\pi_{\bar\theta}(a\mid s)
 \exp\{L_{\bar\theta}(s,a)-L_\theta(s,a)\}}
 {Z_\theta(s)},\quad
 Z_\theta(s)=\mathbb E_{a\sim\pi_{\bar\theta}}
 \exp\{L_{\bar\theta}-L_\theta\}.
 \label{eq:normalized-energy-tilt}
\end{equation}
Then $q_{\bar\theta}=\pi_{\bar\theta}$ and, for any integrable,
$\theta$-independent $A(s,a)$,
\begin{equation}
 \left.\nabla_\theta\mathbb E_{a\sim q_\theta}[A]\right|_{\bar\theta}
 =-\mathbb E_{a\sim\pi_{\bar\theta}}
 \left[(A-\mathbb E[A\mid s])\nabla_\theta L_\theta\right]_{\bar\theta},
 \label{eq:normalized-tilt-gradient}
\end{equation}
so with a conditionally centered advantage,
Eq.~\ref{eq:arcflow-local-gradient} is the exact first-order policy gradient
of a normalized exponential tilt of the behavior policy. The finite clipped
weight in Eq.~\ref{eq:arcflow-weight} inherits no exact likelihood-ratio or
trust-region interpretation.
\end{proposition}

For Proposition~\ref{prop:normalized-tilt}, differentiate
\begin{equation}
 \log q_\theta(a\mid s)=\log\pi_{\bar\theta}(a\mid s)
 +L_{\bar\theta}(s,a)-L_\theta(s,a)-\log Z_\theta(s).
\end{equation}
At $\bar\theta$, $q=\pi_{\bar\theta}$ and
$\nabla\log Z=-\mathbb E_{\pi_{\bar\theta}}[\nabla L]$. Hence
\begin{align}
 \left.\nabla\mathbb E_{q_\theta}[A]\right|_{\bar\theta}
 &=\mathbb E_{\pi_{\bar\theta}}
 \left[A\{-\nabla L+\mathbb E_{\pi_{\bar\theta}}[\nabla L\mid s]\}\right]\\
 &=-\mathbb E_{\pi_{\bar\theta}}
 \left[(A-\mathbb E[A\mid s])\nabla L\right],
\end{align}
which proves Eq.~\ref{eq:normalized-tilt-gradient}.

\paragraph{Energy scale.}
For $\ell=wD$, write $\widehat\Delta_D$ for the probed old-minus-new residual
difference. The local update is proportional to
$-w\,\mathbb E[\widehat A\nabla D]$, while the uncapped likelihood-like
weight is $\exp(w\widehat\Delta_D)$. Thus changing $w$ changes both
the local gradient scale and when the numerical ratio reaches the PPO
clipping thresholds. A fixed $w=1/4$ keeps that scale independent of the
collecting budget. The matched-scale controls isolate this choice in
Appendix~\ref{app:node-placement}.

\subsection{Cross-budget bound}

For $G_K=\{1-j/K:j=0,\ldots,K-1\}$ and
$r(t)=v_\theta(x_t^\star,t\mid s)-u$, define
\begin{equation}
 R_K=\left(\frac1K\sum_{t\in G_K}\|r(t)\|^2\right)^{1/2},\qquad
 \delta_{K,K'}=\max_{t'\in G_{K'}}\min_{t\in G_K}|t'-t|.
\end{equation}

\begin{proposition}[Cross-budget endpoint control]
\label{prop:bridge-cross-budget}
Assume the hypotheses of Proposition~\ref{prop:bridge-endpoint-bound} for a
$K'$-step deployment and additionally that $r(t)$ is $H$-Lipschitz along the
analytic bridge. Then, with $h'=1/K'$,
\begin{equation}
 \|\widehat x_{K'}-a\|
 \le C_{L,K'}\bigl(\sqrt K R_K+H\delta_{K,K'}\bigr),\quad
 C_{L,K'}=h'\sum_{j=0}^{K'-1}(1+h'L)^{K'-1-j}.
 \label{eq:cross-budget-bound}
\end{equation}
Here $\delta_{4,1}=0$ and $\delta_{4,16}=3/16$.
\end{proposition}

For each $t'\in G_{K'}$, choose a nearest $t\in G_K$. The definition of
$R_K$ implies $\|r(t)\|\le\sqrt K R_K$, while time regularity gives
$\|r(t')-r(t)\|\le H\delta_{K,K'}$. Thus every deployment-grid residual is at
most $\sqrt K R_K+H\delta_{K,K'}$. Substitution into
Eq.~\ref{eq:bridge-endpoint-bound}, with $h'=1/K'$, proves
Proposition~\ref{prop:bridge-cross-budget}. Large $H$ or a shifted reference
bridge makes the bound loose, and it is not a reward theorem. Throughout we keep
Euler fixed: the endpoint results assert no invariance to changes of flow
geometry, and state which residuals control one specified finite map.

\FloatBarrier
\section{Experimental protocol and baselines}
\label{app:baseline-audit}

\paragraph{Control training and checkpoint protocol.}
The control run uses $64$ outer iterations, or eight passes through
$\mathcal S$, with eight cached bridge probes per action. It evaluates the
initialization and each cycle-end checkpoint at $\mathcal K=\{1,4,16\}$.
At each checkpoint and budget, $64$ evaluation episodes split into $32$
for selection and $32$ for disjoint reporting. The selection half applies
Eq.~\ref{eq:checkpoint-selection}; exact ties choose the earliest checkpoint.
The report half cannot affect selection and supplies the selected checkpoint's
reported return. On tasks with $1{,}000$-step episodes, the nine evaluations
add $1{,}728{,}000$ environment steps, or $2.75\%$ of the
$62{,}914{,}560$-interaction training budget. This is an empirical deployment
rule, not a formal robustness certificate.

\paragraph{Common-stack comparison.}
FPO is the control backbone. FPO++, DPPO, ReinFlow, and PolicyFlow
are adaptations to the same architecture, initialization, critic, and
interaction budget. Their official releases use other tasks and training
regimes. DPPO originally trains a diffusion policy: it treats stochastic
denoising transitions as an inner MDP and applies policy gradients to their
likelihood ratios \citep{ren2025dppo}. ReinFlow compares against that
diffusion-policy implementation using separately pretrained actors
\citep{zhang2025reinflow}. Our row labeled DPPO instead retains the shared
flow velocity field and applies the denoising-MDP objective to Gaussian-noised
Euler transitions,
$z_{j+1}\mid z_j,o\sim\mathcal N(z_j+\Delta t_j v_\theta(z_j,t_j,o),
\sigma_j^2I)$; the PPO ratios use the corresponding transition log densities.
Thus it tests the DPPO update on a flow backbone, rather than reproducing
the original diffusion actor or its published results. The DM~Control
comparison starts every arm from a shared random initialization, whereas the
original DPPO and ReinFlow studies fine-tune behavior priors. The native
image-based ReinFlow study is
therefore reported separately. The following distinctions define the
baseline rows, rather than implying a reproduction of each paper's native
benchmark results. Table~\ref{tab:baseline-audit} summarizes these adaptations.

\begin{table}[ht]
\centering\small
\caption{Implementation scope of the shared-backbone baselines.}
\label{tab:baseline-audit}
\begin{tabular}{p{0.16\textwidth}p{0.76\textwidth}}
\toprule
Lineage & Adaptation and audited controls \\
\midrule
FPO & Noise-reconstruction surrogate; four-step collection. \\
FPO++ & Per-probe exponential ratios; the shared PPO objective replaces
the release's ASPO objective. \\
DPPO (flow adaptation) & Gaussian-noised Euler transitions and
denoising-MDP likelihood ratios; corrected coordinate reduction and
denoising discount. Trained from scratch on the shared flow backbone. \\
ReinFlow & Stochastic-chain exploration model with the release-style
noise-ceiling schedule restored; shared PPO clipping and backbone. \\
PolicyFlow & Velocity-based surrogate with audited Brownian regularization,
entropy, and optimizer settings. \\
\bottomrule
\end{tabular}
\end{table}

\FloatBarrier
\section{Additional controls}
\label{app:controls}

Table~\ref{tab:mechanism-main} reports the current component ladder. This
appendix adds the collecting-grid isolation used to distinguish energy scale
from node placement.

\subsection{Energy scale and node placement}
\label{app:node-placement}
The collecting-grid controls further separate the time nodes from their
weight. Using weight $1/K$ makes the energy scale depend on the collecting
solver; holding it at $1/4$ removes that dependence.
\begin{figure}[ht]
\centering
\includegraphics[width=\textwidth]{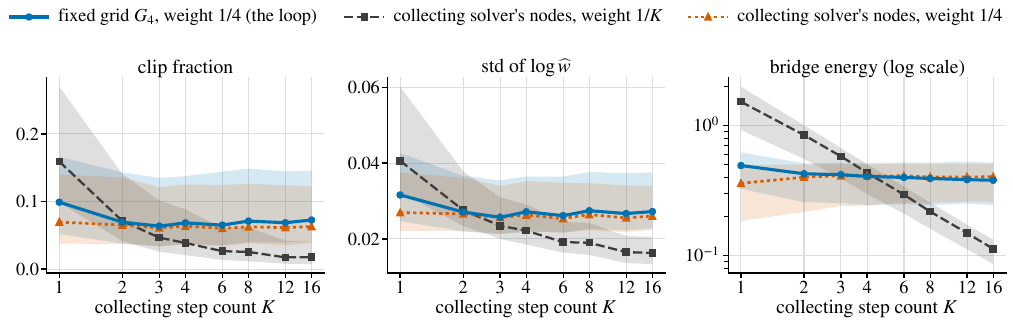}
\caption{\footnotesize \textbf{A fixed energy scale stabilizes the numerical
weight across collecting budgets.} Clip fraction, log-weight spread, and
bridge energy versus collecting $K$; medians and quartiles over the
selection-pool iterations. The collecting-grid $1/K$ weight changes with
budget; both fixed-$1/4$ controls keep the energy scale comparable.}
\label{fig:tilt}
\end{figure}

Figure~\ref{fig:tilt} shows that weight $1/K$ makes the bridge energy fall as
the collecting budget increases, whereas the fixed-$1/4$ variants keep bridge
energy and log-weight spread on comparable scales. This is a scale diagnostic;
it does not by itself compare episodic return.

\end{document}